%% file: sample_paper.tex
\documentclass[twoside]{article}

\usepackage[preprint]{aistats2027}
\usepackage[round]{natbib}

\usepackage{amsmath, amsthm, amssymb}
\usepackage{graphicx}
\usepackage{tikz}
\usepackage{booktabs}
\usetikzlibrary{positioning,chains,fit,shapes,calc,arrows.meta}
\usepackage{textcomp}
\usepackage{url}
\usepackage{pgfplots,wrapfig}
\pgfplotsset{width=10cm,compat=1.9}

\begin{document}

%

%

\twocolumn[

\aistatstitle{Primal-Dual Graph VAEs for Noisy Label Aggregation}

\aistatsauthor{Patrick Stinson \And Nikolaus Kriegeskorte}

\aistatsaddress{Zuckerman Institue\\ Columbia University \And Zuckerman Institue\\ Columbia University} ]

\begin{abstract}
\input{chapters/abstract}
\end{abstract}
\input{chapters/intro}
\input{chapters/figs/high_level}
\input{chapters/preliminaries}
\input{chapters/related}

\input{chapters/method}
\input{chapters/tables/crowdsourcing}
\input{chapters/figs/ablations}
\input{chapters/tables/noisy_labels}

\input{chapters/experiments}
\input{chapters/discussion}
\newpage
\bibliography{refs.bib}



\clearpage
\appendix
\thispagestyle{empty}

\onecolumn
\aistatstitle{Primal-Dual Graph VAEs for Noisy Label Aggregation: \\
Supplementary Materials}
\renewcommand{\thefigure}{A\arabic{figure}}
\setcounter{figure}{0}
\input{chapters/runtime_and_memory}

\end{document}

%% file: chapters/abstract.tex
Inferring the ground-truth from noisy crowdsourced labels is an important theoretical and practical problem. Neural network-based methods offer an alternative to classical Bayesian models which require specifying a family of generative models used for inference. However, current models either still rely on fairly simple generative and inference models or require pseudo-labels or synthetic data to train the aggregate classifier. We propose a graph VAE architecture in which the decoder and encoder use GAT-based message passing on the adjacency graph of a crowdsourced dataset and its dual, respectively. The ground-truth labels are treated as latent variables, enabling unsupervised representation learning without needing to train a separate classifier. We show our model achieves state of the art performance on crowdsourcing benchmarks. We then demonstrate the generality of our approach by showing how the original crowdsourcing graph can be augmented to incorporate side information such as representations from neural network classifiers trained on the noisy labels to substantially boost their classification performance at test time.

%% file: chapters/intro.tex
\section{Introduction}
\label{sec:intro}
Aggregating the estimates from imperfect classifiers into an overall ``meta-classification" is a topic rooted in both statistical learning theory \citep{schapire90strength,freund96experiments} as well as more colloquial statistics, as illustrated by the ``Wisdom of the Crowd" phenomenon famously observed by \citet{galton07vox}. Today, crowdsourcing is a common tool used in a range of applications from forecasting world events \citep{tetlock17expert} to preference modeling for fine-tuning language models \citep{christiano2017deep}.

If the members of the crowd are elicited multiple times for different estimation tasks, simple estimates from the crowd response, such as the majority vote (MV) can be improved upon by parameterizing individual workers' response patterns and modifying the overall estimate based on these parameterizations \citep{dawid79maximum}. Increasingly complex hierarchical Bayesian models \citep{kim12bayesian,li19exploiting,stinson2025nonparametric} can be built to model the population response; however, a key limitation is that these generative models must be handcrafted.

Latent variable models parameterized by deep neural networks (DNNs) \citep{kingma13auto,rezende14stochastic,mnih14neural} eliminates the need for handcrafting. However, crowdsourcing methods based on generative modeling \citep{yin17aggregating} have used only simple transformations to map from the observed crowd response to latent variables and vice versa.

Crowdsourcing data can be represented as a bipartite graph with edges between worker nodes and task nodes representing noisy labels. Recent work has used representation learning on graphs to parameterize worker-task interactions \citep{liu24graph,wu23crowdsourcing}; however, since they lack explicit generative and inference models, a separate classifier must be learned to map the learned representations to ground-truth label estimates. Since the ground-truth labels are not known, pseudo-labels, proxies to the ground-truth labels such as MV, must be used instead and can be of low quality, giving rise to a sub-par classifier.

We propose overcoming the shortcomings of current latent variable models and representation-based models by integrating representation learning in a graph autoencoder \citep{kipf16variational} architecture. Specifically, we use a graph attention network (GAT) \citep{velickovic18graph} to specify a generative model (decoder) for the edges given task latent encodings. The encoder model is a GAT on the dual of the graph in which edge representations are transformed to latent encoding distributions.

In addition, we show how our approach generalizes to integrating disparate sources of information on a graph to do node inference. Specifically, we augment the crowdsourcing graph with a set of nodes corresponding to image representations from DNNs trained on noisy labels and use our method for inference, giving rise to substantially more accurate estimates than those from either source alone.

%% file: chapters/figs/high_level.tex
  \begin{figure*}
  \centering
\includegraphics[width=\textwidth]{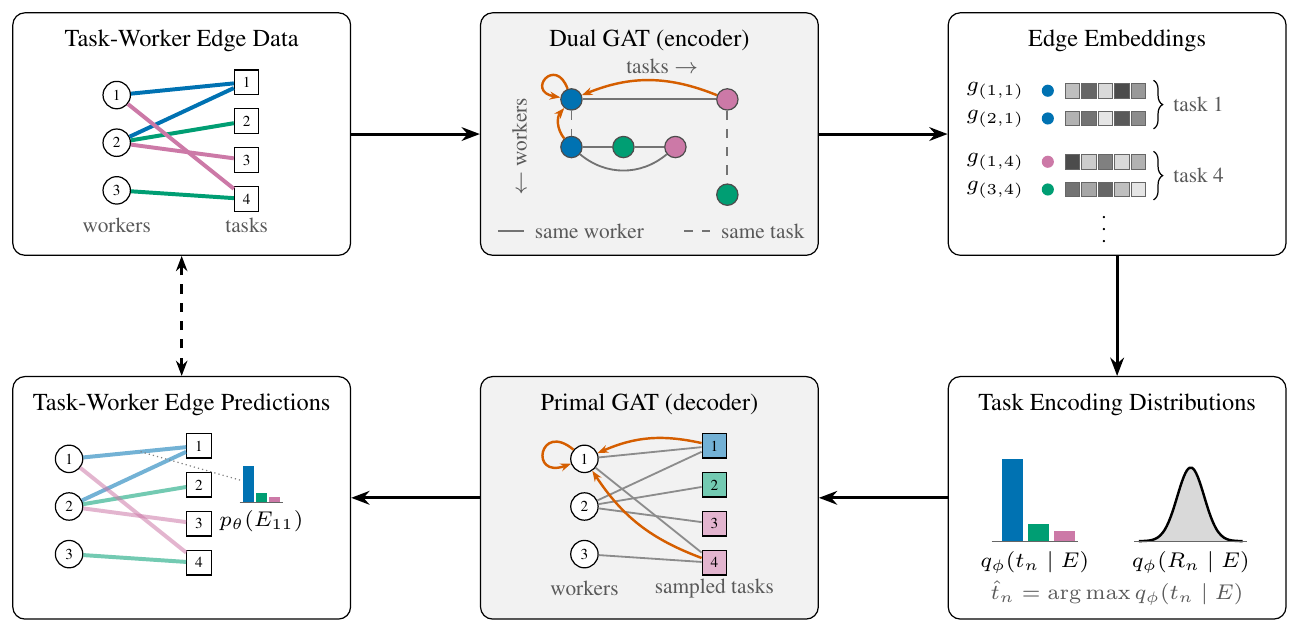}
  \caption{Illustration of our proposed method DPGVAE. Crowdsourced data in the form of worker-task edge labels is input to an edge-encoding (dual) GAT. Edge embeddings partitioned by task are then pooled to produce task node encoding distributions from which task encodings are sampled and input to the primal GAT which does node-level message passing and outputs a predictive distribution over each task-worker labeling. Once our autoencoder is trained, the task encoding distributions serve as the ground-truth label estimates.}
  \label{fig:high_level}
\end{figure*}

%% file: chapters/preliminaries.tex
\section{Preliminaries}
We are given a set of $N$ labeling tasks where each task indexed by $n$ has an unobserved ground-truth value $t^*_n\in[L]$.  $M$ workers each annotate a subset of these tasks, designated as $\mathcal{N}(m)$ for the $m$th worker. The full crowdsourcing dataset can be represented as a bipartite graph consisting of a set of worker nodes indexed by $m\in[M]$ and a set of task nodes indexed by $n\in[N]$ with an edge between worker node $m$ and task node $n$ if worker $m$ labeled task $n$. The value of the edge $E[(m,n)]$ is the (potentially noisy) annotation worker $m$ provided for task $n$. The goal is to estimate the ground-truth task labels $t^*=\{t^*_n\}_{n=1}^N$ from the worker annotations.

To model task-worker interactions we use GAT message passing on the crowdsourcing graph to model the observed edges (annotations) given values for the ground-truth labels which we treat as latent variables $t=\{t_n\}_{n=1}^N$. Inference of the latent ground-truth labels is done using GAT message passing on the dual of the crowdsourcing graph. The GATs on the dual and primal graphs are used as encoder and decoder models, respectively, to form a VAE.

%% file: chapters/related.tex
\section{Related Work}
\subsection{Fully Bayesian approaches}
Bayesian classifier combination \citep{kim12bayesian} generalizes the classical approach from \citet{dawid79maximum} by assuming a generative model of noisy labeling, which is either parameterized by individual workers' confusion matrices (called iBCC for independent) or parameterized by the joint labeling response (called dBCC for dependent). However, the dependent model is generally intractable for datasets involving more than a few workers.

Enhanced Bayesian Classifier Combination (EBCC) \citep{li19exploiting} makes modeling worker dependencies tractable through the use of tensor rank decompositions \citep{hitchcock27expression} to approximate the joint distribution of the population labeling probabilities.
\subsection{Latent encoding-based approaches}
Label-aware autoencoders (LAA) \citep{yin17aggregating} use a VAE \citep{kingma13auto} architecture where (part of) the latent space represents the ground truth labels. However, the encoder consists of simple affine transformations from the space of the entire crowd's response (with missing data represented with zeros) to the latent space, and the decoder is another affine transformation from latent space to the population response space. The simplicity of the encoder and decoder models is a design choice to ensure interpretability of the latent space; however, simple transformations may lack the expressivity needed to model the underlying structure of the dataset.

A similar technique is used in \citet{rodrigues13learning}, where in contrast to crowdsourcing benchmarks, task-level information such as the image data from an image classification task is included. A DNN image classifier is used, but the bottleneck layer represents the logits of the ground-truth labels. Like in LAA, an affine transformation maps the latent space to the population response space.

\subsection{Representation learning-based approaches}
GOVERN \citep{liu24graph} uses data augmentation to create a constrastive loss function \citep{chen20simple,you20graph} using representations formed by a graph convolutional network (GCN) \citep{kipf2017semi}. However, since the ground-truth labels are not explicitly represented, a separate classifier must be trained using pseudo-labels.

TiReMGE \citep{wu23crowdsourcing} uses a multi-view graph representation \citep{hassani20contrastive} of workers and items to minimize the worker-reliability weighted normed difference between workers' labelings and the task representations. However, since the model does not generate edge predictions, there is not a clear way to do cross-validation to select the necessary hyperparameters.

CrowdFM \citep{liu26towards} takes a separate approach by generating a large amount of synthetic data and using attention-based message passing to learn representations that are then decoded into ground-truth label estimates. The model is trained through supervised learning. However, such an approach may build in potentially false inductive biases into the model. 
Additionally, unless one wishes to do inference on a large number of crowdsourcing datasets, training such a model can be prohibitively expensive.

\subsection{Training neural networks with noisy labels}
When task-level information (e.g., images) is available, DNNs can be trained using noisy labels to make task-specific estimates. However, these models are unable to make use of noisy labels at test time.

TAIDTM \citep{li2024transferring} refines the learned representations of confusion matrices of workers using a GCN in conjunction with a task-dependent DNN representation to model instance-dependent confusion matrices and estimate ground-truth labels.

IDNT \citep{guo2023} learns instance-dependent worker confusion matrices by using a Bayesian generalized linear mixed effects model in which the worker-specific and task ground-truth-specific effects are modeled with DNNs and the model is trained using anchor points \citep{liu15classification,patrini17making,xia19anchor}.

%% file: chapters/method.tex
\section{Dual-Primal Graph VAEs}
%
%
We use a VAE with a GNN encoder and decoder and treat the unknown ground truth labels as latent variables. For each item, additional latent variables representing other features of the task can be included. The encoder model is a graph attention network (GAT) \citep{velickovic18graph} on the dual of the crowdsourcing graph; that is, each annotation (edge in the original graph) is given a representation and is adjacent to every annotation from the same worker (giving the node information w.r.t.\ the worker's other annotations) and every annotation of the same task (giving the node information w.r.t.\ the task's other annotations). After message passing, each task's annotation representations are pooled together to determine the posterior over the latent variables. Given samples from the latent distribution, the decoder does message passing between worker and task nodes and predicts annotation logits from a bilinear worker-task score. Our proposed method is illustrated in Figure \ref{fig:high_level}.
\subsection{Inference model}
Given the edge information $E$, the inference model outputs a distribution over ground-truth task labels values $t\in\{1,...,L\}^N$ and task-specific latent variables $R\in\mathbb{R}^{N\times D_r}$:
\begin{equation}
    J|E\sim P(\cdot;GNN_\phi(E)),
\end{equation}
where $J := [t || R]$. The edge embeddings are initialized by $g^{(0)}_{(m,n)} = W^{(\text{embed})}_\text{edge}\text{one-hot}(E[(m,n)],L).$

Each edge embedding $g^{(l)}_{(m,n)}$ has two sets of neighbors: the set of tasks worker $m$ has labeled: $\{g^{(l)}_{m,n'}\}_{n'\neq n}$ and the set of workers rating task $n$: $\{g^{(l)}_{m',n}\}_{m'\neq m}$. Specifying a separate attentional mechanism for these two potentially different types of relationships, we have the update
\begin{gather}
    g^{(c+1)}_{(m,n)}= \|_{k=1}^K f\Bigg(\lambda^{(c,k)}_\text{self}W^{(c,k)}_\text{self}g^{(c)}_{(m,n)} + \nonumber\\
    \lambda^{(c,k)}_{n'\to n}\sum_{n'\in\mathcal{N}(m)\backslash m}\alpha^{(c,k)}_{nn'}W^{(c,k)}_\text{task}g^{(c)}_{(m,n')}+\nonumber\\ 
    \lambda^{(c,k)}_{m'\to m}\left.\sum_{m'\in\mathcal{N}(n)\backslash n}\alpha^{(c,k)}_{mm'}W^{(c,k)}_\text{worker}g^{(c)}_{(m',n)}\right),
\end{gather}
where $c$ and $k$ index layers and attention heads, respectively, and $f$ is a nonlinearity such as ELU \citep{clevert16fast} and attention weights are computed by
\begin{gather}
    \alpha^{(c,k)}_{nn'} =  \frac{\exp(e_{nn'})}{\sum_{\tilde{n}\in\mathcal{N}(m)\backslash n}\exp(e_{n\tilde{n}})},\\
    e_{nn'} = a(g^{(c)}_{m,n},g^{(c)}_{m,n'}),
\end{gather}
where $a$ is an attention mechanism. Message passing along same-task edges $m'\rightarrow m.$ is analogous.  We follow \citet{brody2022how} and set 
\begin{equation}
    a(x,y;W,b) =b^T\text{LeakyReLU}([Wx || Wy]),\label{eqn:attn}
\end{equation}
which enables a more expressive form of attention to that in \citet{velickovic18graph}. Following \citet{velickovic18graph}, to learn robust representations we can apply Dropout \citep{srivastava14dropout} on the embeddings as well as the adjacency list to learn robust relationships between neighbors.

The distribution over ground-truth label values is given by
\begin{equation}
t_n\sim\text{Cat}\left(\sigma\left(W^{(t)}\left\langle g^{(C)}_{(m,n)}\right\rangle_{m\in\mathcal{N}(n)}\right)\right),
\end{equation}
where 
\begin{equation}
\left\langle g^{(C)}_{(m,n)}\right\rangle_{m\in\mathcal{N}(n)} \triangleq\sum_{m\in\mathcal{N}(n)}\tilde{\alpha}_{nm}g^{(C)}_{(m,n)},
\end{equation}
and $\tilde{\alpha}_{mn}$ are attention weights determined through (a separately parameterized) self-attention. 
The distribution over task parameters is given by
\begin{gather}
\hspace{-1cm}R_n\sim\mathcal{N}\left(W^{(r,\mu)}\left\langle g^{(C)}_{(m,n)}\right\rangle_{m\in\mathcal{N}(n)},\nonumber\right.\\ \left. W^{(r,\log\sigma)}\left\langle g^{(C)}_{(m,n)}\right\rangle_{m\in\mathcal{N}(n)}\right).
\end{gather}
\subsection{Generative model}
The generative model outputs a distribution over edges:
\begin{equation}
    E|J \sim P(\cdot;\text{GNN}_\theta(J)).
\end{equation}
The worker and task embeddings are initialized by $h^{(0)}_m = W^{(\text{embed})}_\text{worker}\text{one-hot}(m,M)$ and $j^{(0)}_n = [W^{(\text{embed})}_\text{label} \text{one-hot}(t) || R].$ Updates are given by
\begin{gather}
    \hspace{-1cm}h^{(c+1)}_m = \|_{k=1}^K f\Bigg(\bar{\lambda}^{(c,k)}_\text{self}\bar{W}^{(c,k)}_\text{self}h^{(c)}_m + \nonumber \\ 
    \bar{\lambda}^{(c,k)}_\text{neigh}\sum_{n\in\mathcal{N}(m)}\bar{\alpha}^{(c,k)}_{mn}W^{(c,k)}_\text{neigh}j_n^{(c)}\Bigg),
\end{gather}
with an analogous attention mechanism. Updates for task embeddings $J$ are done in an analogous way. After $C$ layers the GAT outputs the final node embeddings which determine the edge predictive distributions by
    \begin{equation}
    P(E[(m, n)]) =\text{softmax}(({h_m^{(C)}}^TW^{(E)}_l j_n^{(C)})_{l=1}^L),
    \label{eqn:generative}
\end{equation}
where $\{W^{(E)}_l\}_{l=1}^L$ are learnable weight matrices.

For $C=0$, we recover a model that resembles a basic form of collaborative filtering: peer and task parameters are learned without paying attention to the graph structure.
\subsection{Comparison with LAA}
Since both LAA and DPGVAE are VAEs whose latent spaces include the ground-truth labels, comparing their encoders and decoders clarifies the benefits of a graph-based encoder-decoder scheme.

LAA's encoder and decoder are affine transformations mapping from the crowd response to the latent task variables and vice versa. In the encoder, each label's logit is a sum of per-worker contributions determined by separate learned parameters for each worker and label, and a worker's contribution is unaffected by how other workers responded. When $D_r = 0$, this matches the conditional distribution of the true label under the Dawid–Skene model given the workers' confusion matrices. In contrast, after two rounds of message passing in DPGVAE's dual-graph encoder, the embedding for edge $(m, n)$ reflects worker $m$'s other annotations and how other workers annotated those tasks. A vote's weight can therefore depend on agreement patterns, for example discounting workers who consistently vote together. Moreover, the encoder has no worker-specific parameters: reliability is inferred from the voting structure by a function shared across all workers, which can share statistical strength with workers who have annotated only a few tasks and whose separate parameterization would otherwise be noisy.

LAA's decoder models worker labeling logits as an affine function of the task latents, equivalent to a per-worker confusion matrix when $D_r = 0$. Because the map is affine, the non-label latent shifts a worker's logits in the same way regardless of the ground-truth label, so LAA cannot represent, for example, ambiguity that affects only certain classes. For each label, the worker-by-task logit matrix has rank at most $L + D_r$, the maximum number of directions in which tasks can elicit different worker responses. In DPGVAE's primal-graph decoder, because the task embedding is a nonlinear function of $(t_n, R_n)$, the decoder can model interactions between the ground-truth label and the other task latents, and the rank of the logit matrix is bounded by the embedding dimensionality $D_e$. 
\subsection{Optimization}
Following the standard VAE approach \citep{kingma13auto}, we optimize the evidence lower bound (ELBO) given by
\begin{equation}
    \mathcal{L}(\theta,\phi) \triangleq \mathbb{E}_{q_\phi(J|E)}[\log p_\theta(E|J) + \log p_\theta(J) - \log q_\phi(J|E)].
\end{equation} 
The first term measures the quality of the edge reconstructions from the decoder, the second term penalizes the disparity between the task encodings and the prior distribution, and the last term encourages the complexity (entropy) of the task encodings. The expectation over the encoder's latent distribution is intractable so we optimize over a Monte Carlo estimate:
\begin{gather}
\hspace{-1cm}\hat{\mathcal{L}}(\theta,\phi) = \frac{1}{S}\sum_{s=1}^S \log p_\theta(E|J^{(s)}) + \nonumber\\\log p_\theta(J^{(s)}) - \log q_\phi(J^{(s)}|E),
    \label{eqn:mc_elbo}
\end{gather}
where $J^{(s)}\sim q_\phi(\cdot|E).$ 

As a prior over the ground-truth labels we use a tempered version of MV:
\begin{gather}
    \hspace{-1cm}P(t_n=l|\{E[(m,n)]\}_{m\in\mathcal{N}(n)})\propto \nonumber\\\left(\frac{1}{\mathcal{N}(n)}\sum_{m\in\mathcal{N}(n)}\mathbb{I}(E[(m,n)] = l) + \delta\right)^\beta,
\end{gather}
where $\delta$ is an offset ensuring that potential ground-truth labelings without any votes retain some prior probability, and $\beta$ is an inverse-temperature parameter that determines the strength of the MV estimate as the prior. For the non-label task latent variables, we set $p(R_n)=\mathcal{N}(R_n;0,I_{D_r}).$ 

Since the samples $J^{(s)}$ in Equation (\ref{eqn:mc_elbo}) contain discrete valued components, the chain rule cannot be used to estimate the gradient of the parameters controlling the distribution over $t$ in the recognition model. For gradient estimation through discrete-valued sampling we use Reinmax \citep{liu23bridging}, which utilizes Heun’s method \citep{ascher98computer} and achieves second-order estimation accuracy with just gradient information.
\subsubsection{Invariances and subsampling}
The GNN architecture makes the VAE invariant to permuting the task, worker, and/or label indices. The GAT architecture enables subsampling neighbors without needing to rescale the subsampled data since the attention weights already normalize the embeddings.
\subsubsection{Computational Complexity}
\label{sec:complexity}
Memory and compute complexity for using the full graph and its dual are both $\mathcal{O}(D_e\max(M\max_m\mathcal{N}(m)^2,N\max_n\mathcal{N}(n)^2)).$  However, it is straightforward to subsample the graph such that $\max_m\mathcal{N}(m)$ and $\max_n\mathcal{N}(n)$ are bounded to achieve linear scaling in the number of workers and tasks. Dropout can thus play a dual role by providing regularization as well as bounding the memory and computational complexity of our method. We did not need to use Dropout to any of crowdsourcing benchmarks in Section \ref{sec:crowd-expt} for the memory request to fit on our GPU. For our experiments augmenting crowdsourcing graphs with $10,000$ tasks rated by all workers in Section \ref{sec:aug-expt}, a Dropout rate of $0.9$ on the task-linked neighbors of the dual was sufficient. For problems in which the connectivity is bounded, our method can scale with increasing worker (and label) size better than competitors, as demonstrated in Section \ref{sec:runtime_and_memory}.
\input{chapters/figs/augmented_graph}
\subsection{Posterior Collapse and Disentanglement}
\label{sec:collapse}
LAA keeps its encoder and decoder affine so that, with $D_r = 0$, its generative model reduces to the Dawid–Skene latent class model, whose label latent is identifiable as the ground truth under mild conditions \citep{allman2009identifiability}. DPGVAE gives up this guarantee for flexibility. Neither architecture prevents posterior collapse by itself, since even linear VAEs can collapse \citep{lucas19dont}. In DPGVAE, collapse means $q(t_n | E)$ reverts to the tempered MV prior, so its worst case is the MV estimate. The encoding departs from MV on a task when the gain in reconstructing its annotations outweighs the KL cost. Empirically, from Table \ref{tab:crowdsourcing}, DPGVAE outperforms MV on 7 out of 8 datasets, and the label and task latents remain disentangled over training (Figure \ref{fig:mi}).
\input{chapters/figs/mi_fig}

\subsection{Integrating Disparate Sources of Information}
In addition to doing inference on standard bipartite crowdsourcing graphs, our method generalizes to a large family of graphs in which additional multiple sources of information can be included. For instance, in Section \ref{sec:aug-expt}, we first train neural networks using noisy data and use the learned task-specific representations as an extra set of edges connecting to the task nodes, shown in Figure \ref{fig:augmented}. Our method thus generalizes beyond the standard crowdsourcing context in which the data that can be utilized for inference is restricted to a bipartite worker-task graph.

%% file: chapters/figs/augmented_graph.tex
\begin{figure}
  \begin{center}
  \includegraphics[width=\columnwidth]{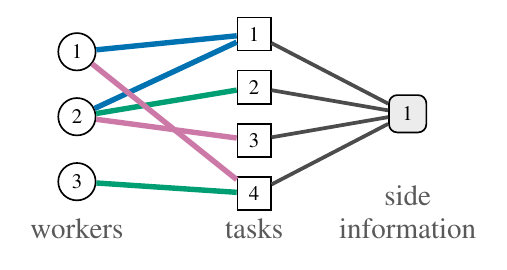}
\caption{Crowdsourcing graph in Figure \ref{fig:high_level} augmented with another source of information, e.g., a DNN image classifier's bottleneck layer representation of each image.}
\label{fig:augmented}
\end{center}
\end{figure}

%% file: chapters/figs/mi_fig.tex
\begin{figure}[!h]
  \begin{center}\includegraphics[width=.48\textwidth]{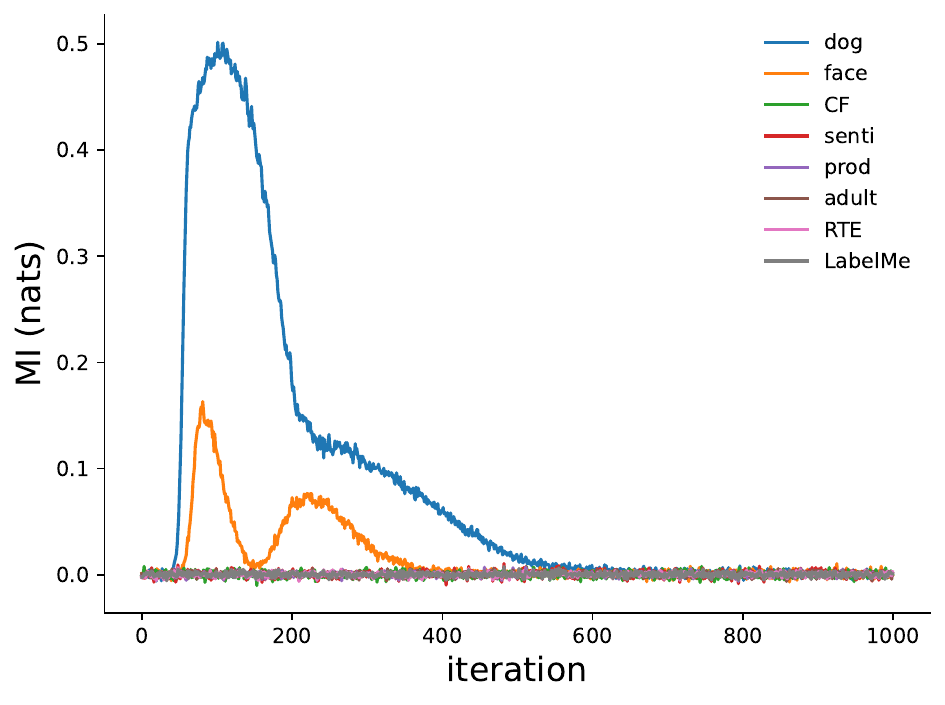}
  \caption{Example traces of sample estimates of $MI(t,R)$ to gauge disentanglement for all crowdsourcing benchmarks in Table \ref{tab:crowdsourcing} (datasets where $5000$ iterations are used are subsampled every $5$ time points). Only two datasets show transient substantial MI, and the model eventually eliminates any MI between $t$ and $R$ (1-dimensional for each task in our experiments) over the course of training.}
  \label{fig:mi}
  \end{center}
\end{figure}

%% file: chapters/tables/crowdsourcing.tex
\begin{table*}[!h]
  \caption{Accuracy of DPGVAE and competing methods on crowdsourcing datasets.}
  \label{tab:crowdsourcing}
  \centering
  \resizebox{2\columnwidth}{!}{
  \begin{tabular}{cccccccccc}
    \toprule
         & CF & face & dog & senti & prod & adult & RTE & LabelMe & avg rank \\
    \midrule
    MV & .880 $\pm$ .006  & .637 $\pm$ .004 & .822 $\pm$ .004  & .933 $\pm$ .002 & .897 $\pm$ .000 & .757 $\pm$ .004 & .899 $\pm$ .003 & .767 $\pm$ .005 & 7.1 \\
    iBCC & .883 $\pm$ .000  & .641 $\pm$ .003  & .831 $\pm$ .001 & .960 $\pm$ .001  & \textbf{.938 $\pm$ .000} & .745 $\pm$ .000 & .921 $\pm$ .001 & .764 $\pm$ .000  & 4.5 \\
    EBCC & .883 $\pm$ .000  & .638 $\pm$ .004 & .840 $\pm$ .000 & \textbf{.961 $\pm$ .000} & .935 $\pm$ .000 & .748 $\pm$ .000 & \textbf{.931 $\pm$ .000} & .786 $\pm$ .000  & 3.3 \\
    GOVERN & .860 $\pm$ .054  & .653 $\pm$ .006 & .822 $\pm$ .001 & .957 $\pm$ .000 & .771 $\pm$ .106 & .610 $\pm$ .094 & \textbf{.931 $\pm$ .000} & .775 $\pm$ .003  & 5.6 \\
    LAA-B & .882 $\pm$ .002  & .652 $\pm$ .001 & .830 $\pm$ .002 & .948 $\pm$ .004 & .921 $\pm$ .002 & .752 $\pm$ .006 & .919 $\pm$ .001 & .758 $\pm$ .002  & 5.2 \\
    LAA-O & .871 $\pm$ .002  & .650 $\pm$ .002 & .832 $\pm$ .002 & .939 $\pm$ .006 & .916 $\pm$ .003 & .756 $\pm$ .003 & .919 $\pm$ .001 & .768 $\pm$ .002  & 5.3 \\
    LAA-L & .873 $\pm$ .003  & .649 $\pm$ .006 & .834 $\pm$ .003 & .937 $\pm$ .004 & .919 $\pm$ .003 & .755 $\pm$ .005 & .917 $\pm$ .001 & .753 $\pm$ .003  & 6.0 \\
    CrowdFM & .881 $\pm$ .000  & .641 $\pm$ .000 & .823 $\pm$ .000 & .882 $\pm$ .000 & .899 $\pm$ .000 & \textbf{.762 $\pm$ .000} & .915 $\pm$ .000 & .771 $\pm$ .000  & 5.9 \\
    \midrule
    DPGVAE & \textbf{.899 $\pm$ .002}  & \textbf{.658 $\pm$ .004} & \textbf{.843 $\pm$ .002} & .956 $\pm$ .003 & .928 $\pm$ .004 & .757 $\pm$ .004 & .925 $\pm$ .002 & \textbf{.788 $\pm$ .003}  & \textbf{2.1} \\
    \bottomrule
  \end{tabular}}
\end{table*}

%% file: chapters/figs/ablations.tex
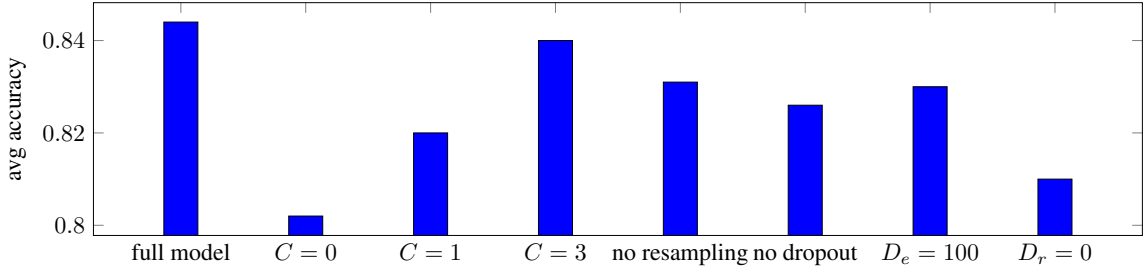
\begin{figure*}
  \centering
\scalebox{.9}{\begin{tikzpicture}
\begin{axis}[
    symbolic x coords={full model, $C=0$,$C=1$,$C=3$, no resampling, no dropout, $D_e=100$, 
    $D_r=0$},width=18.5cm,height=5cm,
    xtick=data,ylabel={avg accuracy}]
    \addplot[ybar,fill=blue,bar width=.5cm] coordinates {
        (full model,.845)
        ($C=0$,.803)
        ($C=1$,.822)
        ($C=3$,.840)
        (no resampling, .832)
        (no dropout, .824)
        ($D_e=100$, .832)
        ($D_r=0$, .808)
    };
\end{axis}
\end{tikzpicture}}
  \caption{Performance of model ablations on the average performance accuracy on crowdsourcing benchmarks in Table \ref{tab:crowdsourcing}. $C$ represents the number of layers in the encoder and decoder GATs, $D_e$ is the fixed embedding dimensionality using in the GATs, 
  and $D_r$ is the encoding dimensionality of the (non-label) task latent variables.}
  \label{fig:ablations}
\end{figure*}

%% file: chapters/tables/noisy_labels.tex
\begin{table*}[!h]
  \caption{Classification accuracy using noisy labels: accuracies of baseline methods compared with their augmentations in DPGVAE.}
  \label{tab:noisy_labels}
  \centering
  \resizebox{2\columnwidth}{!}{
  \begin{tabular}{cccccccccccc}
    \toprule
    & \multicolumn{3}{c}{MNIST} &  \multicolumn{3}{c}{CIFAR-10} & \multicolumn{3}{c}{CIFAR-100} & CIFAR10-N & LabelMe \\
    \midrule
    label noise& low & mid & high & low & mid & high & low & mid & high \\
    \midrule
    DPGVAE & .926 $\pm$ .003 & .777 $\pm$ .002 & .616 $\pm$ .004 & .937 $\pm$ .003 & .798 $\pm$ .001 & .645 $\pm$ .003 & .601 $\pm$ .002 & .574 $\pm$ .002 & .534 $\pm$ .006 & .911 $\pm$ .004 & .790 $\pm$ .004\\
    \midrule
    IDNT & .990 $\pm$ .001 & .988 $\pm$ .001 & .979 $\pm$ .002 & .880 $\pm$ .001 & .874 $\pm$ .001 & .870 $\pm$ .001 & .573 $\pm$ .001 & .557 $\pm$ .004 & .512 $\pm$ .007 & .898 $\pm$ .003 & .850 $\pm$ .010\\
    DPGVAE $+$ $\hat{y}_\text{IDNT}$ & .973 $\pm$ .002 & .893 $\pm$ .003 & .782 $\pm$ .004 & .973 $\pm$ .001 & .883 $\pm$ .003 & .782 $\pm$ .004 & .658 $\pm$ .002 & .635 $\pm$ .005 & .600 $\pm$ .005 & .948 $\pm$ .003  & .847 $\pm$ .003 \\
    DPGVAE $+$ $\hat{x}_\text{IDNT}$ & \textbf{.996 $\pm$ .001} & \textbf{.992 $\pm$ .002} & \textbf{.989 $\pm$ .002} & \textbf{.959 $\pm$ .004} & \textbf{.925 $\pm$ .003} & \textbf{.883 $\pm$ .004} & \textbf{.733 $\pm$ .003} & \textbf{.714 $\pm$ .003} & \textbf{.677 $\pm$ .005} & \textbf{.949 $\pm$ .003}  & \textbf{.875 $\pm$ .013}\\
    \midrule
    TAIDTM & .993 $\pm$ .001 & .993 $\pm$ .002 & .962 $\pm$ .003 & .912 $\pm$ .002 & .909 $\pm$ .002 & .899 $\pm$ .001 & N/A & N/A & N/A & .900 $\pm$ .004 & .854 $\pm$ .011 \\
    DPGVAE $+$ $\hat{y}_\text{TAIDTM}$ & .971 $\pm$ .001 & .896 $\pm$ .002 & .773 $\pm$ .002 & .971 $\pm$ .002 & .893 $\pm$ .001 & .790 $\pm$ .001 & N/A & N/A & N/A & .831 $\pm$ .002 & .822 $\pm$ .004\\
    DPGVAE $+$ $\hat{x}_\text{TAIDTM}$  & \textbf{.998 $\pm$ .000} & \textbf{.996 $\pm$ .001} & \textbf{.971 $\pm$ .003} & \textbf{.973 $\pm$ .004} & \textbf{.944 $\pm$ .003} & \textbf{.911 $\pm$ .002} & N/A & N/A & N/A & \textbf{.945 $\pm$ .003} & .\textbf{876 $\pm$ .009}\\
    \bottomrule
  \end{tabular}}
\end{table*}

%% file: chapters/experiments.tex
\section{Experiments}
\label{sec:experiments}
Experiments were carried out on a cloud instance using an NVIDIA H100 GPU. 

After each training epoch, we calculated the reconstruction loss on a held out part of the data (10\% of the dataset) as a validation set. After the specified number of training epochs, the model with the lowest validation reconstruction loss was used as the final model. The argmax of the task label latent encoding distributions resulting from performing inference on the full dataset were used as the final ground-truth label estimates. Accuracies are shown as mean $\pm$ standard deviation over five random seeds, using random tie-breaks over MV estimates. 
We used the official implementations for LAA, EBCC, GOVERN, CrowdFM, IDNT, and TAIDTM and used \cite{li19exploiting}'s implementation of iBCC.

For all experiments, the MV prior tempering hyperparameters $\delta$ and $\beta$ were set to be .01 and 1.0, respectively, and we used $C=2$ layers in our GATs and set $D_r=1$ in all non-ablation experiments. We used Adam \citep{kingma15adam} with a cosine-annealed learning rate starting at $5\times 10^{-4}$. In the crowdsourcing benchmark experiments in Section \ref{sec:crowd-expt}, we performed cross-validation over hyperparameters: embedding Dropout $\{0, .2, .5\}$, neighbor Dropout: $\{0,.2,.5\}$, embedding dimensionality $\{20,50,100\}$, number of attention heads $\{5,10\}$ number of epochs $\{1000,5000\}$. For our experiments in Section \ref{sec:aug-expt}, we did not do cross-validation over the hyperparameters and instead fixed the embedding Dropout to .5, the primal GAT dropout to .5, the dual GAT dropout to .9, the embedding dimensionality to 30, the number of attention heads to 5, and the number of training epochs to 5000.
\subsection{Crowdsourcing benchmarks}
\label{sec:crowd-expt}
We first compared our proposed method with competitors on crowdsourcing benchmarks. Datasets include \texttt{CF} \citep{josephy14workshops}: sentiment classification on tweets about the weather, \texttt{face} \citep{mozafari14scaling}: facial expression classification, \texttt{dog} \citep{zhou2012learning}: labeling of different dog breeds, \texttt{senti} \citep{zhang14spectral} tweet sentiment classification for a company, \texttt{prod} \citep{kraska12crowder}: determining whether two descriptions are for the same product, \texttt{adult} \citep{ipeirotis2010quality}: age-restriction appropriateness of a website given its url, \texttt{RTE} \citep{snow08cheap}: textual entailment classification, \texttt{LabelMe} \citep{rodrigues2018} image scene classification. LAA-B, LAA-O, and LAA-L refer to variants of LAA detailed in \citep{yin17aggregating} that correspond to models with $D_r=0,1,2,$ respectively.

In Table \ref{tab:crowdsourcing}, we show our model performs competitively to superior to the competing models. Overall, we find its ranking average over all datasets to be the best.
\subsubsection{Ablations}
We then performed ablations on our model to gain further insights into which parts of our model were important to its success. We show the accuracy of the ablated models averaged over all datasets in Table \ref{tab:crowdsourcing} in Figure \ref{fig:ablations}.

Our ablations show that while cross-validation helps (as expected given the heterogeneity of datasets), there is no catastrophic loss if a given element is fixed. 
The most notable decrease in performance comes from $C=0$, which corresponds to not using the GAT architecture to encode or decode the latent variables and instead use affine transformations to map from latent space to embedding space in the encoder and vice versa in the decoder. There is little to no performance drop increasing from $C=2$ layers to $C=3$. 

Resampling and Dropout are shown to be beneficial to performance, but removing them does not incur catastrophic loss, showing the robustness of our method. 
%
Fixing the embedding dimensionality to the highest among the hyperparameters, $D_e=100$ also incurs only a slight drop in performance, suggesting that it is fairly difficult to overparameterize the model when the encoding and decoding are done by GATs (additionally, resampling and Dropout can help regularize the model).

Setting $D_r=0$ decreases the performance of the model more substantially. Without this extra latent encoding information, it may be difficult for the model to distinguish among tasks with the same inferred ground-truth label. For example, in a dataset in which each worker labels each task, embeddings among tasks with the same task label latent variable will be identical.
\subsection{Training DNNs with noisy labels and augmenting the crowdsourcing graph with their representations}
\label{sec:aug-expt}
We then wanted to see if our model could do a more generic form of crowdsourcing by integrating labeling information from workers with representations from neural networks, illustrated in Figure \ref{fig:augmented}, as DNNs themselves cannot make use of noisy labels at test time, nor can previously proposed crowdsourcing methods make use of neural network representations.
\subsubsection{Simulated labelings}
We first generated simulated labelings of three different noise levels: low, mid, and high with an average labeling error rate of 20\%, 35\%, and 50\%, respectively, for MNIST \citep{deng12mnist}, CIFAR-10, and CIFAR-100 \citep{krizhevsky09learning} following the procedure outlined in \citep{guo2023}. Each task received 5 labelings out of the pool of 10 simulated workers. We then trained IDNT \citep{guo2023} and TAIDTM \citep{li2024transferring} classifiers from the noisy labels using their official implementations. For MNIST, we used Lenet-5 \citep{lecun98gradient} and ResNet-18 \citep{he16deep} for CIFAR-10 and CIFAR-100.

Once the models are trained, we take the representations before the last fully connected and softmax layers (with dimensionalities of 84 and 512 for Lenet-5 and ResNet-18, respectively) and use them as the initial embeddings for the edges between the task nodes and the side information node in Figure \ref{fig:augmented}. We otherwise train our model in the same way using a dual GAT as an encoder and a primal GAT as the decoder with the primal graph as represented in Figure \ref{fig:augmented}.

In Table \ref{tab:noisy_labels}, we show that by using the augmented graph and combining the information in the noisy labelings with the image information encoded by the DNNs (referred to as $\hat{x}$), we achieve a much higher performance over just using the trained DNN (IDNT or TAIDTM in Table \ref{tab:noisy_labels}) or just using our method with the noisy labelings (DPGVAE in Table \ref{tab:noisy_labels}). As a baseline, we also compare the performance gain by simply including the DNN's label estimate ($\hat{y}$ in Table \ref{tab:noisy_labels}) and find that while helpful, this alone does not account for the substantial performance boost.
\subsubsection{Real world datasets}
We used the same training and evaluation procedure for two datasets with real world labelings. CIFAR-10N \citep{wei22learning} consists of three worker labelings for each of the 50000 images in CIFAR-10's training set. LabelMe \citep{rodrigues17learning,torralba10labelme} consists of 1000 $256\times 256\times 3$ images labeled by an average of 2.5 workers per image. We randomly split each dataset into an 80/10/10 train/validation/test partition. For CIFAR-10N, the architecture was ResNet-18, and for LabelMe, as in \citet{guo2023}, each image was first passed through a pretrained VGG-16 \citep{simonyan15very} network followed by two trainable fully-connected layers (the first with 50\% Dropout) followed by a softmax classification layer. For the VGG-16 network, the representation given to DPGVAE was 128 dimensional.

In Table \ref{tab:noisy_labels} we see the same qualitative trends for the real world noisily labeled datasets as the simulated datasets, where DPGVAE done over the crowdsourcing graph agumented the DNN representations performs substantially better than the neural network models on their own or models using the crowdsourced labels on their own, with the more accurate DNN estimates $\hat{y}$ not sufficient to account for the performance boost. We found CIFAR-100 to be prohibitively slow for TAIDTM.

%% file: chapters/discussion.tex
\section{Discussion}
In the first half of this paper, we have proposed a crowdsourcing model that combines unsupervised learning and representation learning to overcome shortcomings in previous crowdsourcing models using either only unsupervised learning or representation learning. In the second half of this paper, we have shown how our proposed model can naturally generalize to incorporate disparate sources of information relevant for noisy label aggregation. Namely, we have shown how our model can integrate both noisy labels and DNN image representations in order to make substantially more accurate predictions than when just using one source of information, demonstrating a form of human-machine complementarity \citep{steyvers22bayesian}.

%% file: chapters/runtime_and_memory.tex
\section{Runtime and Memory Experiments}
\label{sec:runtime_and_memory}
In this section, we provide experimental results to demonstrate DPGVAE's ability to scale well to the large label setting. Additionally we show that for a dataset with more workers and fewer total annotations per worker, we achieve far lower memory demands with much faster runtime whereas IDNT's memory demands and runtime increase dramatically.

\begin{table}[h]
\caption{Runtime and memory for DPGVAE and IDNT for different problem sparsities. ``w." refers to the number of workers and ``annot." to the number of annotations per task.} \label{tab:runtime_and_memory}
\begin{center}
\resizebox{1\columnwidth}{!}{
\begin{tabular}{llll}
& \textbf{CIFAR-10 (10 w. 5 annot.)}  &\textbf{CIFAR-100 (10 w. 5 annot.)} & \textbf{CIFAR-100 (100 w. 5 annot.)} \\
\hline \\
IDNT memory & 4.5 G & 4.8 G & 17.3 G \\
IDNT runtime & 2.6 hours & 5.1 hours & 13 hours \\
DPGVAE + $\hat{x}_{IDNT}$ memory & 48.5 G & 53.0 G & 6.0 G \\
DPGVAE + $\hat{x}_{IDNT}$ runtime & 1.5 hours & 1.5 hours & 7 minutes \\
\end{tabular}}
\end{center}
\end{table}

This result is consistent with our complexity analysis in the paper where memory and computational complexity grow linearly with the number of workers and tasks but quadratically with the maximum number of tasks annotated by a single user and the maximum number of workers annotating a single item. As mentioned in the paper limiting these quantities through Dropout can increase performance in addition to giving lower memory and runtime requirements.

Since the neural network embeddings
are 512-dimensional, the memory and runtime requirements for DPGVAE + $\hat{x}$
serve as an upper bound for DPGVAE+$\hat{y}$
(and DPGVAE on the original crowdsourcing graph).

We used the same hyperparameters in the paper for our experiment for consistency. However, if a practitioner needed to further reduce the memory footprint or runtime of our method, Dropout could be used more aggressively